\documentclass[twoside,11pt]{article}
\usepackage{subcaption}
\usepackage[preprint]{jmlr2e-arxiv}
\usepackage[utf8]{inputenc}
\usepackage{dsfont}
\usepackage{natbib}
\usepackage{graphicx,graphics}
\usepackage{float}
\usepackage{wrapfig}
\usepackage{amsmath}
\usepackage{enumitem}
\usepackage{booktabs}
\usepackage{nicefrac}
\usepackage{yfonts}
\usepackage[euler]{textgreek}
\usepackage[normalem]{ulem}
\usepackage{listings}
\usepackage{MnSymbol}
\usepackage{silence}
\usepackage[mathscr]{eucal}
\usepackage{algorithmicx,algorithm}

\makeatletter
\def\thm@space@setup{%
  \thm@preskip=\parskip \thm@postskip=0pt
}
\makeatother

\usepackage{algpseudocode}

\makeatletter
\newcommand{\algmargin}{\the\ALG@thistlm}
\makeatother
\algnewcommand{\parState}[1]{\State\parbox[t]{\dimexpr\linewidth-\algmargin\relax}{\strut #1\strut}}
\allowdisplaybreaks

\newcounter{parentalgorithm}

\def\E{\mathbb{E}}
\def\R{\mathbb{R}}

\newcommand{\argmax}{\mathop{\mathrm{argmax}}}

\usepackage{xcolor}
\usepackage{tikz}
\usetikzlibrary{arrows.meta,positioning}
\usepackage[textsize=tiny]{todonotes}

\newcounter{counter}[section]

\newcommand{\conf}[1]{#1}
\newcommand{\rev}[1]{#1}
\newcommand{\mot}[1]{#1}
\newcommand{\abl}[1]{#1}

\newcommand{\PTR}{\mathrm{PTR}}
\newcommand{\logit}{\operatorname{logit}}

\newcommand{\Var}{\operatorname{Var}}

\newcommand{\Ccal}{\mathcal{C}}

\newcommand{\Scal}{\mathcal{S}}
\newcommand{\eff}{\mathrm{eff}}
\newcommand{\transp}{^{\!\top}}

\begin{document}

\title{Exploratory Creative Pre-Ranking:\\ Hierarchical Empirical-Bayes Thompson Sampling\\ for Real-Time Creative Optimization}

\author{\name Jiale Han\footnotemark[1] \email jialehan@ucla.edu\\
\addr Department of Statistics and Data Science, UCLA
\AND
\name Ajith Moparthi
\email amoparthi@roblox.com\\
\addr Roblox
\AND
\name Dan Xie
\email dxie@roblox.com\\
\addr Roblox
\AND
\name Xiaowu Dai\footnotemark[2]
\email daix@ucla.edu\\
\addr Departments of Statistics and Data Science and of Biostatistics, UCLA\\
}

\maketitle
\renewcommand{\thefootnote}{\fnsymbol{footnote}}
\footnotetext [1] {work done during an internship at Roblox.}
 \footnotetext [2] {\textit{Address for correspondence:} Xiaowu Dai, Department of Statistics and Data Science, UCLA, 8125 Math Sciences Bldg \#951554,  Los Angeles, CA 90095, USA. Email: daix@ucla.edu.}

\begin{abstract}
Dynamic creative optimization decides which creative of an ad to render for each request. Recent work places that choice before ad ranking, but relies on deterministic point predictors with no calibrated uncertainty. AI-generated content has made creatives nearly free, while feedback per creative remains sparse and selection faces a tight latency budget. We introduce a creative \emph{pre-ranker} that scores every recalled creative and shortlists each campaign's pool before the per-user ranker runs. Creatives sit in a hierarchy of audience segment, advertised item, campaign and creative, so a cold creative falls back on what its parent levels have earned, and how much it borrows from its parents is set in closed form from historical logs, with one parameter per level and nothing tuned. To stay within the latency budget, the pre-ranker predicts for an audience \emph{segment} rather than an individual user, from per-attribute counters. It selects by a single Thompson draw per creative, balancing exploration and exploitation so that a new creative still has a chance to be learned. Replaying Roblox Sponsored-ads traffic over campaigns with competing creatives, the shortlist lifts predicted play-through rate by $8.14\%$ against $2.80\%$ for an equally sized random shortlist, forwards more than $40\%$ fewer creatives than the full pool, and adds only $0.2\%$ to serving time. After an A/B test on live traffic, the method now serves in production at Roblox.
\end{abstract}


\section{Introduction}

Online advertising platforms decide, for every request, which ads to show and how to render each one. A single advertised item, which we call a \emph{universe} (a \emph{game} on Roblox, the platform whose traffic we evaluate on, or a product on an e-commerce site), carries many \emph{creatives}: distinct combinations of images, titles and videos (Figure~\ref{fig:creatives}). Which creative is displayed has a first-order effect on engagement and hence on revenue \citep{hill2017efficient, lin2022joint, yang2024parallel}, and selecting the most effective one in real time and per user is \emph{dynamic creative optimization} (DCO). With AI-generated content advertisers now produce creatives at negligible cost, to the point that generation is itself optimized for click-through rate \citep{chen2025ctr}, so the creative pool is far larger than the ad pool while the feedback any single creative receives is extremely sparse: \rev{an individual creative-by-segment cell gathers only a handful of impressions, which is why creative selection is posed as a cold-start or bandit problem rather than a supervised one \citep{zhao2019you, wang2021hybrid, chen2021automated}.} DCO therefore sits between two competing requirements: the estimator must be \emph{statistically efficient} under sparse, cold-start feedback, and \emph{cheap enough} to evaluate inside the request over a candidate set far larger than the ad pool.

\begin{figure}[h] \centering \begin{subfigure}[t]{0.49\linewidth} \centering \includegraphics[width=0.85\linewidth]{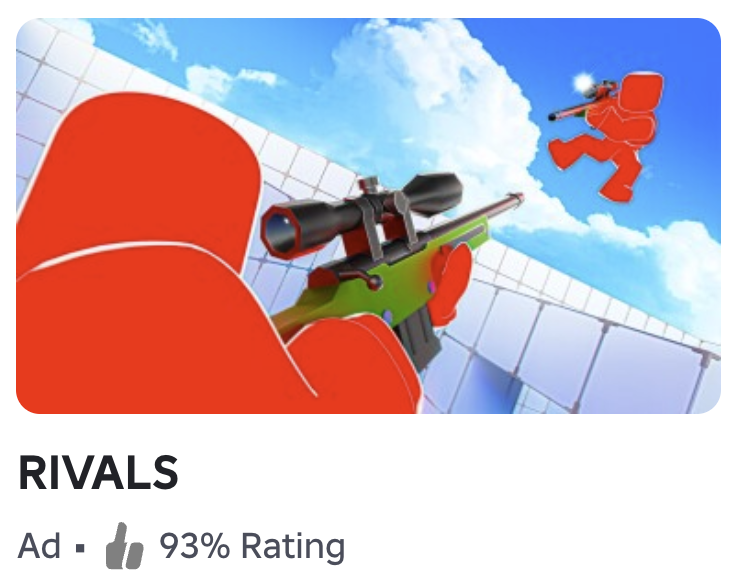} \end{subfigure} \begin{subfigure}[t]{0.49\linewidth} \centering \includegraphics[width=0.85\linewidth]{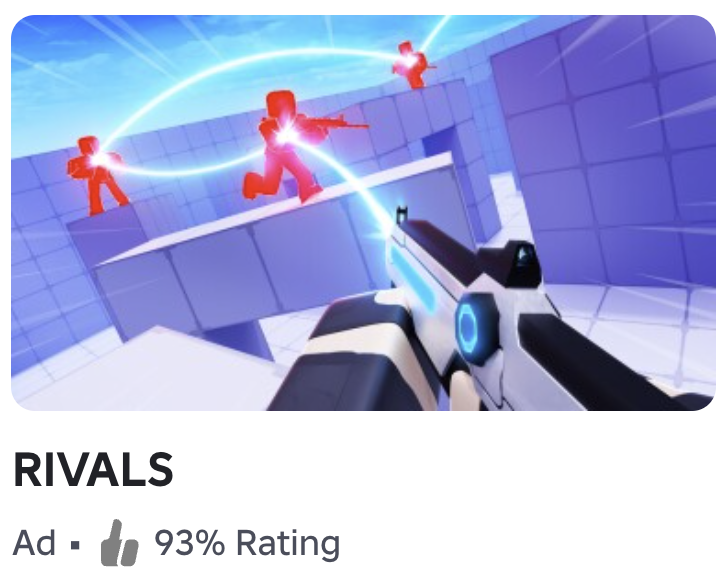} \end{subfigure} \caption{\conf{Two different creatives of a single advertised item.}}  \label{fig:creatives} \end{figure}

\mot{The usual answer is to order the stages so that creative selection can afford to be expensive. Ad ranking runs first, and the creative is chosen \emph{afterwards} for the handful of ads that have already won. That set is small enough to score per user. The saving costs information. The ranker scores each ad on a default rendering rather than on what the user will be shown, and a creative that would have made an ad competitive cannot rescue an ad the ranker has already dropped. Recent industrial work reverses the order.} \citet{lin2022joint} place the creative selector \emph{before} the ranker in a cascade, training a two-tower selector \citep{huang2013learning} with list-wise distillation from the ranking stage and sharing its creative embeddings with that stage; \citet{yang2024parallel} instead run the two modules in parallel, keeping the creative side cheap by giving it far fewer features than the ad ranker. These systems establish that the ordering should put creatives first, and we take that as settled: with the rendering fixed before ranking, the ranker consumes the chosen creative's features rather than a default proxy, and both report online gains on production traffic. \mot{What the reversal leaves open is what this stage does to learning.} Both are \emph{point predictors} of click-through rate with no calibrated uncertainty, so their selection is deterministic: a creative is surfaced because the model already scores it highly, and a fresh one is reached only through some separate mechanism.

The stage they presuppose is the one we study, and its position in the funnel is what makes exploration necessary rather than optional. A \emph{creative pre-ranker} sits between recall and ad ranking, reducing each recalled campaign's pool to an \emph{exploratory} shortlist and passing only that shortlist to a heavy per-user ranker whose cost grows with the set it receives. Exploratory is a strong requirement. The shortlist need not contain the creative the pre-ranker currently rates highest. When it leaves that creative out, the ranker must choose among the rest, and an untried creative gets shown. That is how it earns the evidence to overturn the current ranking. Ads commonly reach the ranker through a stage of this kind \citep{wang2020cold, ma2021fscd, li2022autofas, li2022inttower}; creatives commonly do not, although the creative pool is the larger of the two. Because the same stage that decides what the ranker sees also decides which creatives ever accumulate evidence, a deterministic pre-ranker forecloses its own learning \citep{wei2024enhancing}: a creative it never shortlists is a creative it can never learn about. Yet a pre-ranker is worth having only if it is far cheaper than the ranker it protects, which rules out per-user prediction on two counts: a single user's history is far too sparse to support a reliable per-creative estimate, and scoring every recalled creative per user is exactly the cost the stage exists to remove. This leads to our central question:

\emph{Can a creative pre-ranker be made statistically efficient enough to learn from sparse, cold-start feedback, and simultaneously cheap enough to score every recalled creative within the request?}

We answer affirmatively. The pre-ranker we build is deliberately coarser than the ranker it protects, predicting for an audience segment rather than an individual user, which is cheap enough to evaluate over the entire recalled set and leaves user-level resolution to the ranker downstream. Within that budget it keeps a calibrated sense of what it does not yet know, so a cold creative borrows strength from the item it advertises and the audience it targets, and that same uncertainty is what makes it explore. Our contributions are threefold.

(i) \textbf{A hierarchy the data calibrates, with no knobs to tune.} Creatives sit in a hierarchy (platform average, audience segment, advertised item, campaign, creative), so one with little history leans on its parents and trusts its own record only as that record grows. We show the squared-error-minimizing amount of leaning is one number per level---the impressions at which a creative's own record has earned as much weight as its parent's---read off historical logs in closed form, with nothing tuned (Section~\ref{sec:mse}).

(ii) \textbf{An audience-level estimate from counters cheap enough to keep.} Tracking the engagement of every creative in every combination of \rev{categorical attributes} is impractical online. The combinations multiply while each is starved of data. We keep one counter per attribute instead, dense and cheap, and combine them at request time into an estimate for the exact audience served. The combining rule is fitted offline on the full attribute combinations, which batch pipelines can see even though serving cannot maintain them. It is exact in two limiting cases, when the attributes carry independent information and when they merely repeat one another (Section~\ref{sec:realtime}).

(iii) \textbf{A shortlist that is nearly free and cuts the ranker's load.} Each creative is scored by one draw from its posterior, and each campaign gets a share of the shortlist proportional to its pool size, filled to a global budget by round robin, so no advertiser gains exposure by uploading more creatives (Section~\ref{sec:ts}). On replayed Roblox traffic it recovers $44\%$ of the gain from forwarding every creative on more than $40\%$ fewer of them, lifts predicted play-through rate by $8.14\%$ against $2.80\%$ for a random shortlist of equal size, and adds only $0.2\%$ to serving time (Section~\ref{sec:exp}). After an A/B test on live traffic, the method now serves production at Roblox.

\subsection{Related Work}

\noindent\textbf{Creative optimization and ranking in advertising.} Creative selection has largely been cast as a bandit problem, over layouts \citep{schwartz2017customer, hill2017efficient} and over composited elements \citep{chen2021automated, chen2021efficient}, all optimizing the creative in isolation, downstream of ad ranking; cold start is elsewhere primed with visual content \citep{zhao2019you, wang2021hybrid} or with user--creative profiles \citep{zhang2025culcnet}. Two industrial systems reverse that ordering, selecting creatives before \citep{lin2022joint} or beside \citep{yang2024parallel} the ranker, but both select deterministically, with no calibrated uncertainty to explore; sparsity they meet by distilling from the ranker and leaning on content features. We take the ordering as settled and contribute the stage it presupposes, an exploratory creative \emph{pre-ranker} that shortlists each campaign's pool from a hierarchical posterior over marginal counters. It can run independently ahead of both to provide more exploration and shorten the pool their selectors see.

\noindent\textbf{Empirical Bayes and shrinkage.} Borrowing strength across related units by shrinking per-unit estimates toward a common target is classical, from the James--Stein estimator \citep{james1961estimation} to parametric empirical Bayes \citep{efron1975data, morris1983parametric}. We use these ideas as the backbone of a \emph{servable} multi-level prior, fitting each level's variance component in closed form with the impression-weighted moment estimator of \citet{dersimonian1986meta} specialized to Binomial data, which gives a tuning-free concentration per level. Our contribution is not shrinkage itself but its instantiation as a hierarchical prior that a production system can evaluate under a latency budget, refreshed in real time by the marginal-counter reconstruction.

\noindent\textbf{Thompson sampling and hierarchical bandits.} Thompson sampling \citep{thompson1933likelihood, chapelle2011empirical} enjoys strong Bayesian and frequentist regret guarantees for Bernoulli and contextual bandits \citep{agrawal2012analysis, kaufmann2012thompson, russo2014learning, daniel2018tutorial}, and sharing a prior across related arms sharpens exploration in the many-armed, cold-start regime \citep{kveton2021meta, hong2022hierarchical}. Our selector is a Beta--Bernoulli Thompson sampler whose prior is the calibrated hierarchy above: its parameters are learned offline and its posterior is refreshed online from lightweight counters, matching the production separation between a slowly drifting prior and fast real-time evidence. We contribute the hierarchical, marginal-reconstructed posterior being sampled, not a new regret analysis. \section{Problem Setting}\label{sec:setting}

\begin{figure*}[t] \centering
\begin{tikzpicture}[
  x=1mm, y=1mm, >={Latex[length=1.5mm]}, font=\footnotesize,
  stage/.style={draw, rounded corners=1.5pt, align=center, text width=19mm, minimum height=9mm, inner sep=2pt},
  ours/.style={stage, fill=black!10, thick},
  tag/.style={draw, dashed, rounded corners=1.5pt, align=center, text width=26mm, font=\tiny, inner sep=2pt},
  cnt/.style={font=\scriptsize, inner sep=1pt},
  camp/.style={draw, rounded corners=1.5pt, minimum height=8mm, inner sep=1pt},
  keep/.style={circle, draw, fill=black!60, minimum size=2.6mm, inner sep=0pt},
  drop/.style={circle, draw, fill=white, minimum size=2.6mm, inner sep=0pt},
  faint/.style={dashed, black!50}
]
\node[stage] (rec)  at (0,0)   {Recall\\{\scriptsize campaigns matching $s$}};
\node[ours]  (pre)  at (36,0)  {Creative\\ pre-ranking\\{\scriptsize\textbf{this work}}};
\node[stage] (rank) at (80,0)  {Ad ranking\\with creative\\selection};
\node[stage] (auc)  at (122,0) {Auction and render};
\draw[->] (rec) -- node[cnt,above]{$\sum_g|\Ccal_g|$} (pre);
\draw[->] (pre) -- node[cnt,above]{${\le}\,B$} (rank);
\draw[->] (rank) -- node[font=\tiny,inner sep=1pt,above,align=center]{${\le}\,1$ creative\\per universe} (auc);

\node[tag] (off) at (22,17) {offline prior: concentrations $\kappa$, pooling weights $\psi$, node rates};
\node[tag] (rt)  at (54,17) {real-time counters $n_{c,k},y_{c,k}$};
\draw[->] (off.south) -- ([xshift=-5mm]pre.north);
\draw[->] (rt.south) -- ([xshift=5mm]pre.north);
\draw[->] (auc.north) |- (rt.east);
\node[font=\tiny, anchor=south] at (100,17.6) {logged impressions and plays};

\draw[faint] (4,-9) rectangle (112,-28);
\draw[faint] (pre.south west) -- (4,-9);
\draw[faint] (pre.south east) -- (112,-9);
\node[camp, minimum width=20mm] at (30,-16) {};
\node[camp, minimum width=20mm] at (58,-16) {};
\node[camp, minimum width=13mm] at (86,-16) {};
\foreach \x in {23,26.5} \node[keep] at (\x,-16) {};
\foreach \x in {30,33.5,37} \node[drop] at (\x,-16) {};
\foreach \x in {51,54.5} \node[keep] at (\x,-16) {};
\foreach \x in {58,61.5,65} \node[drop] at (\x,-16) {};
\node[keep] at (82.5,-16) {};
\foreach \x in {86,89.5} \node[drop] at (\x,-16) {};
\node[cnt] at (30,-22) {$E_g{=}2$};
\node[cnt] at (58,-22) {$E_g{=}2$};
\node[cnt] at (86,-22) {$E_g{=}1$};
\node[font=\tiny, anchor=west, align=left] at (6,-25.5) {one draw per creative;\\ top $E_g$ kept per campaign};
\node[keep] at (66,-25.5) {}; \node[font=\tiny, anchor=west] at (67.8,-25.5) {shortlisted};
\node[drop] at (88,-25.5) {}; \node[font=\tiny, anchor=west] at (89.8,-25.5) {dropped};
\end{tikzpicture}
\caption{The creative pre-ranker in context. It scores every recalled creative with one draw from its hierarchical posterior, assembled offline and refreshed online from single-attribute counters, and shortlists the top $E_g=\max(1,\lfloor\rho|\Ccal_g|\rfloor)$ of each campaign (inset) up to a global budget $B$ filled by round robin. Only that shortlist reaches ad ranking, which scores it with the production creative-selection model and keeps at most one creative per advertised item. The cap $B$ controls its per-candidate cost.}  \label{fig:pipeline} \end{figure*}
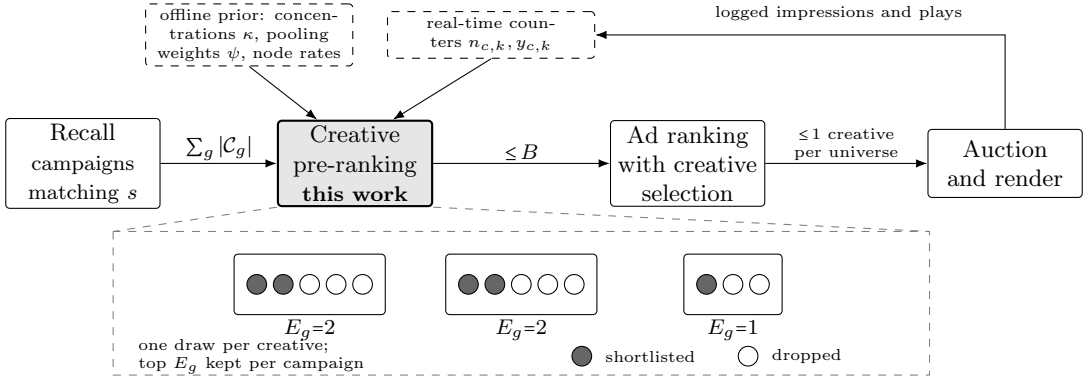

\subsection{Creative Selection Before Ad Ranking}\label{sec:system}

When a user arrives on the platform, the request creates advertising impression opportunities contested by every advertiser whose targeting matches that user. Targeting is expressed through categorical attributes carried by the request (e.g., location), which we summarize as a \emph{segment} $s=(d_1,\dots,d_K)$, a vector of $K$ discrete attributes, and the ad system first \emph{recalls} all eligible advertisements whose targeting matches $s$, the opening stage of a multi-stage retrieval cascade \citep{liu2017cascade}. Each advertised item, or \emph{universe} $u$, may run several \emph{campaigns} $g$ that differ in objective and budget, and each campaign contains many \emph{creatives} $c$: the distinct renderings (image, title, video) that constitute the individual ads actually shown. Carrying every candidate creative into the downstream ranker inflates request latency and forgoes exploration, whereas collapsing each ad to campaign-level information discards precisely the creative-level signal that drives engagement. What is needed is a lightweight model that \emph{ranks creatives} ahead of ad ranking. Prior creative-ranking systems optimize click-through rate \citep{lin2022joint, yang2024parallel}; we study its counterpart on Roblox, the \emph{play-through rate} (PTR), the probability that a user enters the advertised game.

Our goal is an \emph{exploratory creative pre-ranker} (Figure~\ref{fig:pipeline}): a stage between recall and ad ranking that reduces each recalled campaign's pool to a \emph{shortlist} and passes only that to the ranker. The common reduction for creatives is to pick one per campaign by deterministic point prediction \citep{lin2022joint, yang2024parallel, zhang2025culcnet}: a creative is surfaced because the model already ranks it highly, so the rest of the pool never accumulates the evidence that might overturn that ranking. Being at the same time far cheaper than the ranker it protects, ours predicts PTR at the granularity of a segment and leaves user-level resolution downstream. We write $\theta_{c,s}\in(0,1)$ for the true PTR of creative $c$ in segment $s$.

\subsection{A Hierarchical Random-Effects Model of PTR}\label{sec:hier}

We model PTR hierarchically, so that a leaf with little data can always draw on information from its parents. Each node of the hierarchy (the global root, the segment $s$, the universe $u$, the campaign $g$, and the creative $c$) has its own true rate, and we decompose $\theta_{c,s}$ into one mean-zero shift per level, following a nested (hierarchical) ANOVA,
\begin{equation}\label{eq:anova}
    \theta_{c,s}\;=\;\mu_0+a_s+b_u+e_g+h_c,
\end{equation}
where $\mu_0=\E[\theta_{c,s}]$ is the global mean PTR and each remaining term is a node's shift from its parent's conditional mean, from $a_s=\E[\theta_{c,s}\mid s]-\mu_0$ down to $h_c=\theta_{c,s}-\E[\theta_{c,s}\mid g,s]$, with variances $\sigma_s^2,\sigma_u^2,\sigma_g^2,\sigma_c^2$; by construction each term averages to zero over its parent's traffic. The decomposition mirrors the containment already built into the ad system, so that a sparsely served node can lean on its parent's estimate instead of its own thin data. Conditioned on a true rate $\theta$, a node's plays are Binomial,
\begin{equation}\label{eq:binom}
    y\mid\theta\sim\mathrm{Binomial}(n,\theta),
\end{equation}
where $n$ is impressions, $y$ plays, $f=n-y$ non-plays, and $\hat p=y/n$ the raw rate, each carrying whichever subscript matches its level (e.g.\ $n_{c,s}$ at the leaf, $n_u$ at the universe).

\section{Main Results} Section~\ref{sec:mse} derives the MSE-optimal estimator of $\theta_{c,s}$ under the ideal assumption that full play and impression counts are available at every level. Section~\ref{sec:realtime} then reconstructs that estimate from single-attribute counters, which are far cheaper to maintain and query at request time, learning how to combine them with a $K$-parameter logistic regression fit offline rather than assuming the targeting attributes independent. Section~\ref{sec:ts} turns the resulting posterior into a selection policy via Thompson sampling. 

\subsection{Empirical-Bayes Shrinkage with a Data-Calibrated Concentration}\label{sec:mse} A node with few impressions cannot be read off its raw rate $\hat p=y/n$, and the hierarchy exists precisely so that such a node can borrow strength from its parent. We therefore shrink the node's own rate toward its parent's true mean, written $\mu$,
\begin{equation}\label{eq:shrink}
\hat\theta=w\,\hat p+(1-w)\,\mu,\qquad w\in[0,1],
\end{equation}
trusting the data more as $n$ grows. The only question is how to set the weight $w$, which we do from the data with no free parameter. Consider a parent node with children $[C]=\{1,\dots,C\}$ (e.g.\ the campaigns in a universe), each child $i$ having a true PTR $\theta_i$, and posit across children the moment model
\begin{equation}\label{eq:remodel}
\begin{gathered}
\E[\theta_i]=\mu,\qquad \Var(\theta_i)=\sigma^2,\\
y_i\mid\theta_i\sim\mathrm{Binomial}(n_i,\theta_i),
\end{gathered}
\end{equation}
under which $\hat p_i=y_i/n_i$ is conditionally unbiased for $\theta_i$ with exact sampling variance $\theta_i(1-\theta_i)/n_i$. This constrains only the first two moments of $\theta_i$ and imposes no distributional form on its prior.

\begin{theorem}[MSE-optimal shrinkage]\label{thm:mse} Under model~\eqref{eq:remodel}, among all shrinkage estimators of the form~\eqref{eq:shrink}, the weight minimizing the mean squared error (MSE) $\E[(\hat\theta_i-\theta_i)^2]$ is
\begin{equation}\label{eq:wstar}
\begin{gathered}
    w^\star=\frac{\sigma^2}{\sigma^2+\big(\mu(1-\mu)-\sigma^2\big)/n_i}
    =\frac{n_i}{n_i+\kappa},\quad
    \kappa=\frac{\mu(1-\mu)}{\sigma^2}-1 ,
\end{gathered}
\end{equation}
and the resulting estimator is $\hat\theta_i=(\kappa\mu+y_i)/(\kappa+n_i)$. \end{theorem}

The proof is in Appendix~\ref{app:mse}. This estimator is exactly the posterior mean of a $\mathrm{Beta}(\kappa\mu,\kappa(1-\mu))$ prior under the Binomial likelihood~\eqref{eq:binom}, so we adopt the Beta as the working prior for every node. It reproduces the MSE-optimal shrinkage at no loss, since the matched prior has mean $\mu$ and variance exactly $\sigma^2$, and being conjugate it keeps the posterior Beta, $\theta\mid y_i\sim\mathrm{Beta}(\kappa\mu+y_i,\kappa(1-\mu)+f_i)$, so updates are incremental and posterior sampling is trivial for the Thompson-sampling policy \citep{chapelle2011empirical}. The crossover $w^\star=\tfrac12$ falls at $\kappa=n_i$: the concentration is the number of a node's own impressions at which it trusts its data and its prior equally.

Two population quantities remain. The parent mean $\mu$ we estimate directly by the impression-pooled rate $\hat\mu=\big(\sum_i y_i\big)/\big(\sum_i n_i\big)$. The between-child variance $\sigma^2$ is the crux, because it is not directly observable: the spread we can measure across children mixes the true between-child spread with sampling noise, since $\Var(\hat p_i)=\sigma^2+\E[\theta_i(1-\theta_i)]/n_i$. Separating the two requires comparing siblings, for which we assume the following.

\begin{assumption}[Conditionally independent children]\label{ass:re} Conditionally on the parent, the children are independent: the pairs $(\theta_i,y_i)$, $i\in[C]$, are mutually independent, with $\theta_i$ (mean $\mu$, variance $\sigma^2$) and $y_i\mid\theta_i\sim\mathrm{Binomial}(n_i,\theta_i)$ as above. \end{assumption}

The assumption is mild here, because the hierarchy is built precisely so that the systematic co-movement among siblings is carried by the parent: shared audience, placement and temporal effects enter through the common $\mu$ and $\sigma^2$, and conditioning on the parent removes them, leaving each child its own random effect together with Binomial counts over that child's disjoint stream of impressions.

Under Assumption~\ref{ass:re}, the impression-weighted scatter $Q=\sum_{i} n_i(\hat p_i-\hat\mu)^2$ of a parent's children satisfies
\begin{equation}\label{eq:scatter}
\begin{split}
\E[Q]={}&(C-1)\big(\mu(1-\mu)-\sigma^2\big)+\sigma^2\Big(N-\frac{\sum_{i=1}^{C} n_i^2}{N}\Big),
\end{split}
\end{equation}
with $N=\sum_i n_i$ (Theorem~\ref{thm:scatter}, Appendix~\ref{app:scatter}). This is what makes $\sigma^2$ identifiable: the expected scatter is a known-coefficient linear function of $\sigma^2$, a signal term $\sigma^2\big(N-\sum_i n_i^2/N\big)$ plus a sampling-noise offset $(C-1)\big(\mu(1-\mu)-\sigma^2\big)$, so subtracting the offset and dividing by the signal coefficient inverts the observed scatter. Two gaps keep it from being an estimator on its own: it holds only in expectation, so a single parent's scatter is a high-variance readout, and it still involves the unknown $\mu$. Pooling the scatters of all parents at a level, which their independence lets concentrate, and substituting $\hat\mu$ for $\mu$ closes both.

\begin{corollary}[Moment estimator for the level variance]\label{cor:sigma} Under Assumption~\ref{ass:re} at each parent of a level with parents $[P]=\{1,\dots,P\}$, let parent $j$ have $C_j$ children with impressions $n_{ji}$, total $N_j=\sum_{i=1}^{C_j} n_{ji}$, pooled rate $\hat\mu_j=\big(\sum_{i=1}^{C_j} y_{ji}\big)/N_j$, and scatter $Q_j=\sum_{i=1}^{C_j} n_{ji}(\hat p_{ji}-\hat\mu_j)^2$. With $D_j=N_j-\sum_{i=1}^{C_j} n_{ji}^2/N_j-(C_j-1)$, the estimator
\begin{equation}\label{eq:sigmahat}
\hat\sigma^2=\max\!\Bigg\{\frac{\sum_{j=1}^{P}\big[\,Q_j-(C_j-1)\,\hat\mu_j(1-\hat\mu_j)\,\big]}{\sum_{j=1}^{P} D_j},\ 0\Bigg\}
\end{equation}
estimates the level variance $\sigma^2$ with only an $O(1/N_j)$ bias per parent: it is exactly unbiased when the true means $\mu_j$ are used in place of $\hat\mu_j$ (and before clipping). \end{corollary}

The proof is in Appendix~\ref{app:sigma}. This is the DerSimonian--Laird moment estimator \citep{dersimonian1986meta} specialized to Binomial data with the exact sampling variance. The $(C-1)$ correction is what removes the sampling-noise bias: a parent with a single child, or with one impression per child, has $D=0$ and is correctly ignored rather than letting pure sampling noise masquerade as between-child signal.

Corollary~\ref{cor:sigma} supplies one variance component per layer. For the prior mean we may fit one value per layer or, since PTRs are small and vary little across layers, adopt the global rate $\mu_0$. Either choice keeps only one prior strength $\kappa=\mu(1-\mu)/\sigma^2-1$ per layer. Writing $\kappa_s,\kappa_u,\kappa_g,\kappa_c$ for those strengths, we apply Theorem~\ref{thm:mse} once per layer of~\eqref{eq:anova}, each layer's estimate becoming the mean the layer beneath it shrinks toward, starting from $\mu_0$ at the top:
\begin{equation}\label{eq:ladder}
\begin{aligned}
    \mu_s&=\frac{\kappa_s\mu_0+y_s}{\kappa_s+n_s}, &
    \mu_{u}&=\frac{\kappa_u \mu_s+y_{u}}{\kappa_u+n_{u}},\\
    \mu_{g}&=\frac{\kappa_{g} \mu_{u}+y_{g}}{\kappa_{g}+n_{g}}, &
    \mu_{c,s} &= \frac{\kappa_c \mu_{g}+y_{c,s}}{\kappa_{c}+n_{c,s}},
\end{aligned}
\end{equation}
and the posterior distribution of the PTR is
\begin{equation}\label{ptr}
\begin{gathered}
    \theta_{c,s}\sim\mathrm{Beta}(\alpha_{c,s},\beta_{c,s}), \\
    (\alpha_{c,s},\beta_{c,s})=\big(\kappa_c\mu_g+y_{c,s},\ \kappa_c(1-\mu_g)+f_{c,s}\big),
\end{gathered}
\end{equation}
with MSE-optimal point estimate $\mu_{c,s}$. Every line of~\eqref{eq:ladder} is the estimator $\hat\theta=(\kappa\mu+y)/(\kappa+n)$ of Theorem~\ref{thm:mse} read at one layer. We name it $\mu$ with that layer's subscript rather than $\hat\theta$, keeping $\theta$ for a true rate and $\mu$ for the shrunken estimate of it that the next layer down treats as its prior mean.

\subsection{A Real-Time Posterior from Marginal Counters}\label{sec:realtime} Evaluating the posterior~\eqref{ptr} through the cascade~\eqref{eq:ladder} requires, for a given segment $s$, the counters $(n,y,f)$ at every level from the segment down to the creative. In production these are read from marginal aggregates that may be sparse, or altogether absent for a fresh node, so we back off along the hierarchy: whenever a level has no impressions of its own, its estimate defaults to its parent's, cascading upward until, in the extreme, it reduces to the global rate $\mu_0$,
\begin{equation}\label{eq:fallback}
\begin{aligned}
 \mu_{c,s}&\xrightarrow{\,n_{c,s}=0\,} \mu_g\xrightarrow{\,n_{g}=0\,}\mu_u\xrightarrow{\,n_{u}=0\,} \mu_s \xrightarrow{\,n_{s}=0\,} \mu_0 ,
\end{aligned}
\end{equation}
so the prior mean in \eqref{ptr} is $\mu_g$ when the campaign has support, and otherwise the one inherited from the finest coarser level that has support. We write \eqref{ptr} in that form, with $\mu$ that inherited mean, as
\begin{equation}\label{ptr2}
\begin{gathered}
    \theta_{c,s}\sim\mathrm{Beta}(\alpha_{c,s},\beta_{c,s})=\mathrm{Beta}(\kappa_c\mu+y_{c,s},\ \kappa_c(1-\mu)+f_{c,s}).
\end{gathered}
\end{equation}

These joint counts are hard to maintain in real time, and not only at the leaf: the cascade needs the segment rate at \emph{every} level. With $K$ attributes and up to $m$ values each, a segment is one of as many as $m^K$ combinations, so a joint counter per node-by-segment cell is infeasible to store and search online, and each such cell is severely sparse. What stays cheap and dense are \emph{single-attribute} counters, impressions and plays split by one attribute at a time, only $mK$ of them per node. The reconstruction is identical at every level, so we state it for a generic node $z$ (a creative, a campaign, a universe, or the global root). For attribute $k$ and value $v$ we keep impressions $n_{z,k}(v)$ and plays $y_{z,k}(v)$, and we write $\hat p_z=y_z/n_z$ for the node's overall rate, the totals being independent of $k$ because each attribute's buckets partition the node's impressions. Throughout, $\logit x=\log\frac{x}{1-x}$ and $\operatorname{sigmoid}(t)=(1+e^{-t})^{-1}$ is its inverse. A bucket with few impressions is noisy for the reason of Section~\ref{sec:mse}, so we first stabilize each marginal by shrinking it toward the node's own rate, with per-attribute strength $\kappa_k$ from Corollary~\ref{cor:sigma},
\begin{equation}\label{eq:mshrink}
\begin{gathered}
    \bar p_{z,k}(v)=\frac{\kappa_k\hat p_z+y_{z,k}(v)}{\kappa_k+n_{z,k}(v)},\quad
    w_{z,k}(v)=\frac{n_{z,k}(v)}{\kappa_k+n_{z,k}(v)},
\end{gathered}
\end{equation}
and summarize it by its \emph{lift} over the node on the log-odds scale,
\begin{equation}\label{eq:lift}
    \delta_{z,k}(v)=\logit\bar p_{z,k}(v)-\logit\hat p_z .
\end{equation}
A request's segment $s=(d_1,\dots,d_K)$ selects one lift per attribute, so the model has to turn the $K$ numbers $\delta_{z,1}(d_1),\dots,\delta_{z,K}(d_K)$ into the joint rate $\theta_z(s)$. We assume it does so linearly, with one coefficient per attribute.

\begin{assumption}[Pooling of marginals]\label{ass:pool} There are weights $\psi=(\psi_1,\dots,\psi_K)$, common to all nodes. For every node $z$ and segment $s$,
\begin{equation}\label{eq:pool}
\begin{split}
\logit\theta_{z}(s)
  ={}&\logit\hat p_z+\sum_{k=1}^{K}\psi_k\,\delta_{z,k}(d_k).
\end{split}
\end{equation}
\end{assumption}

Read as a statistical model, \eqref{eq:pool} \emph{is} a logistic regression for the play indicator of an impression: its covariates are the $K$ lifts the counters report, $\logit\hat p_z$ is a fixed offset, and $\psi$ is the coefficient vector. Unit weights,
\begin{equation}\label{eq:naive}
\begin{gathered}
\psi\equiv\mathbf 1,\qquad\text{i.e.} \ \logit\theta_{z}(s)\approx\logit\hat p_z+\sum_{k=1}^{K}\delta_{z,k}(d_k),
\end{gathered}
\end{equation}
recover the standard additive model \citep{richardson2007predicting,hastie1986generalized} and are exact when the attributes are independent in the node's traffic, since each bucket then measures its own attribute's effect and nothing else. \rev{Targeting axes co-vary in practice rather than being independent}, so a bucket that looks strong is partly strong because of the company it keeps, and unit weights credit the same signal more than once. Two extremes make this precise: if the attributes are mutually independent then \eqref{eq:pool} holds with $\psi\equiv\mathbf 1$, and if $J$ of them duplicate one another while the remaining $K-J$ are independent then it holds with $\psi_k=1/J$ on the duplicated group and $\psi_k=1$ elsewhere (Proposition~\ref{prop:limits}, Appendix~\ref{app:limits}). Between them, $\psi_k$ says how much of a bucket's apparent deviation is genuinely its own, and $1/\psi_k$ reads as an effective number of attributes sharing that signal. Real attributes fall in between, which makes $\psi$ an empirical question rather than something to assume; estimating it instead of asserting it is our only departure from \eqref{eq:naive}.

The weights are shared across levels, used at the creative for its real-time rate and at the coarser nodes for the prior mean that rate is shrunk toward, and are fit once on universe-level cells while $\hat p_z$ and $\delta_{z,k}$ are re-estimated per node. Sharing is deliberate: the weights say how far the targeting attributes duplicate one another in the traffic, which is a property of the request stream rather than of the node, and they are identified only where cells are dense, so fitting them per creative would buy a small reduction in bias at a much larger variance. It also keeps the posterior coherent, since prior mean and real-time rate are combined with weight $\kappa_c/(\kappa_c+n_\eff)$, which would lose the MSE reading of Theorem~\ref{thm:mse} were they reconstructed differently. Empirically, weights fit on universe cells and applied to held-out creative cells cut the reconstruction's mean absolute error by about a third relative to \eqref{eq:naive}.

The joint cell is unavailable \emph{online} but perfectly visible \emph{offline}, since the binding constraint is the online store and not the historical record. We therefore fit $\psi$ against batch cells, once per refresh window and alongside the concentrations of Section~\ref{sec:mse}, by maximizing the ridge-penalized Binomial log-likelihood of \eqref{eq:pool},
\begin{equation}\label{eq:psifit}
\begin{split}
\hat\psi=\argmax_{\psi_1,\dots,\psi_K\in\R}\ \sum_{z,s}\Big\{&y_z(s)\log \pi_{z,s}(\psi)\\
&\hspace{-5.2em}+\big(n_z(s)-y_z(s)\big)\log\big(1-\pi_{z,s}(\psi)\big)\Big\}-\frac{\lambda}{2}\,\|\psi\|_2^2,
\end{split}
\end{equation}
where $n_z(s)$ and $y_z(s)$ are the cell's impressions and plays, $\pi_{z,s}(\psi)=\operatorname{sigmoid}\big(\logit\hat p_z+x_z(s)\transp\psi\big)$ for the covariate row $x_z(s)=\big(\delta_{z,1}(d_1),\dots,\delta_{z,K}(d_K)\big)$ read from the same stored marginals serving uses, and $\hat\psi$ is clipped to $[0,1]^K$, the range in which Proposition~\ref{prop:limits} gives it meaning. Nothing is fit at request time, and the fit is cheap enough to live in the aggregation pipeline because it is $K$-dimensional however many cells the window holds (Appendix~\ref{app:psifit}).

Replacing $\psi$ by $\hat\psi$ in \eqref{eq:pool} gives the estimator evaluated at request time,
\begin{equation}\label{eq:recon}
\hat\mu_{z}=\operatorname{sigmoid}\Big(\logit\hat p_z+\sum_{k=1}^{K}\hat\psi_k\,\delta_{z,k}(d_k)\Big).
\end{equation}
Instantiating \eqref{eq:recon} at $z=c$ gives the leaf estimate $\hat\mu_{c,s}$, and at the campaign, the universe and the global root the higher-level estimates that are the serving-time counterparts of the cascade priors in \eqref{eq:ladder} and the fallbacks in \eqref{eq:fallback}. Two mechanisms discount a marginal, answering different questions: the shrinkage \eqref{eq:mshrink} asks whether a bucket has enough data to be believed, while $\hat\psi_k$ asks whether what it says has already been said. Both degrade gracefully, since an attribute with little support has vanishing shrinkage weight, so its lift drops out and the estimate falls back to $\logit\hat p_z$.

At the leaf the rate alone does not determine the posterior: by \eqref{ptr2} the update requires the effective counts, not just the point estimate. Their support cannot be the sum $\sum_k n_{c,k}(d_k)$, since the buckets $d_1,\dots,d_K$ all contain the same impressions and summing counts each of them as many as $K$ times. What the marginals do supply is a bound: every impression in the cell lies in each of those buckets, so $n_{c,s}\le n_{c,k}(d_k)$ for every $k$. We take the support to be that bound,
\begin{equation}\label{eq:neff}
    n_{\eff}=\min\big\{\,n_{c,k}(d_k)\;:\;1\le k\le K,\ n_{c,k}(d_k)>0\,\big\},
\end{equation}
with $n_{\eff}=0$ when no attribute has support, so a never-served attribute is ignored and the support is set by the scarcest one actually used. Any larger figure would credit the cell with more evidence than the least-served attribute defining it could supply. Together, $n_{\eff}$ and $\hat\mu_{c,s}$ give the leaf's effective play and no-play counts,
\begin{equation}\label{eq:effcounts}
    y_{c,s}=n_{\eff}\,\hat\mu_{c,s},\qquad f_{c,s}=n_{\eff}\,(1-\hat\mu_{c,s}).
\end{equation}

The prior strength $\kappa_c$, the weights $\hat\psi$ and the prior mean $\mu$ all drift slowly and are computed offline. Since $\kappa_c$ does not depend on the segment we estimate it from the full-segment counters, and $\hat\psi$ is a single $K$-vector fetched once per request; the prior mean does depend on the segment, so to keep it cheap to fetch under a single lookup key we precompute the single-attribute counters and form $\hat p_z$ and $\delta_{z,k}(d_k)$ in advance. The only work left online is the $K$-term weighted sum \eqref{eq:recon} and one Beta draw per candidate. Substituting \eqref{eq:effcounts} and the offline prior into \eqref{ptr} yields the serve-time posterior, whose mean shrinks $\hat\mu_{c,s}$ toward $\mu$ by the factor $\kappa_c/(\kappa_c+n_{\eff})$ and reduces to $\mu$ exactly when no marginal supports the cell.

\subsection{Exploratory Creative Pre-Ranking}\label{sec:ts} Sections~\ref{sec:mse}--\ref{sec:realtime} give, for every candidate creative $c$ and the request segment $s$, a serve-time posterior with effective counts from \eqref{eq:effcounts},
\begin{equation}\label{eq:servepost}
\begin{gathered}
\alpha_{c,s}=\kappa_c\mu+n_{\eff}\,\hat\mu_{c,s},\\
\beta_{c,s}=\kappa_c(1-\mu)+n_{\eff}(1-\hat\mu_{c,s}).
\end{gathered}
\end{equation}
We turn each posterior into a score by \emph{Thompson sampling} \citep{thompson1933likelihood,chapelle2011empirical}: for every candidate we draw
\begin{equation}\label{eq:draw}
\tilde\theta_{c,s}\sim\mathrm{Beta}(\alpha_{c,s},\beta_{c,s}),
\end{equation}
independently across candidates and afresh on every request, and take that draw as the creative's score. Scoring by a draw rather than by the posterior mean makes exploration automatic and proportional to uncertainty: a fresh creative has a diffuse posterior and is occasionally sampled high, so it is tried, while a well-measured under-performer has a tight posterior and is rarely shown. On cold start the posterior collapses to the hierarchical prior $\mathrm{Beta}(\kappa_c\mu,\kappa_c(1-\mu))$, so exploration begins from the best available fallback \eqref{eq:fallback} rather than from an uninformative guess, and because the draws use the segment-specific parameters~\eqref{eq:servepost} the exploration is itself segment-aware: a creative may be well determined on the traffic it has received and still be explored on an audience it has not yet reached.

Pre-ranking is then immediate. Every creative carried past this point is one the downstream ranker must score with a heavy per-user model, so what the system commits to is, for each recalled campaign $g$ with creative pool $\Ccal_g$, a shortlist of the $E_g$ creatives with the largest draws,
\begin{equation}\label{eq:topk}
\begin{gathered}
\Scal_g=\argmax_{\Scal\subseteq\Ccal_g,\ |\Scal|=E_g}\ \sum_{c\in\Scal}\tilde\theta_{c,s},\quad
E_g=\max\big(1,\lfloor\rho\,|\Ccal_g|\rfloor\big),
\end{gathered}
\end{equation}
so a campaign's quota is a fixed fraction $\rho\in(0,1]$ of the creatives it has uploaded: a large pool, which needs more exploring, receives a longer shortlist, while the floor at one keeps every recalled campaign represented. The fraction is the only knob, and it trades ranker cost against how much of the pool survives.

What the latency budget really constrains is the total number of creatives reaching the ranker, so the platform also caps that total at $B$. When $\sum_g E_g$ exceeds $B$, we fill $B$ by \emph{round robin}: campaigns take one slot each on the first pass, then cycle for further slots while they stay below quota, until $B$ is exhausted or no campaign wants more. This is the cheap way to make the cap fair: it costs $O(B)$ with no global sort or cross-campaign score comparison, and no advertiser gains share by uploading a larger pool. Within a campaign the slots are still filled by the highest draws, so how the budget is divided does not affect which creatives a campaign explores.

With $E_g=1$ the rule is one round of Thompson sampling in a Beta--Bernoulli bandit whose arms are the campaign's creatives, so the classical regret guarantees apply campaign by campaign \citep{agrawal2012analysis,daniel2018tutorial}, and $E_g>1$ is the multiple-play analogue. Our contribution is not a bandit analysis but the posterior being sampled: hierarchical, reconstructed from marginal counters, and cheap enough to sample for every creative of every recalled campaign, which is what lets a pre-ranker explore at all. Algorithm~\ref{alg:ts} collects the whole method.

\begin{algorithm}[!ht]
\caption{Exploratory Creative Pre-Ranking by Hierarchical Empirical-Bayes Thompson Sampling.}\label{alg:ts}

\begin{algorithmic}[1]
\parState{\textbf{Input:} Request segment $s=(d_1,\dots,d_K)$, recalled campaigns $g$ with creative pools $\Ccal_g$, quota fraction $\rho\in(0,1]$, latency cap $B$.}
\State \textbf{Offline} (batch pipeline, once per refresh window):
    \parState{Fit the level concentrations $\kappa_s,\kappa_u,\kappa_g,\kappa_c$ and the per-attribute $\kappa_k$ from aggregated logs by \eqref{eq:sigmahat}.}
    \parState{Fit the pooling weights $\hat\psi$ on joint node-by-segment cells by \eqref{eq:psifit}, clipped to $[0,1]^K$.}
    \parState{Publish one row per node $z$ above the creative: its rate $\hat p_z$ and shrunk marginal lifts $\delta_{z,k}(v)$ from \eqref{eq:mshrink}--\eqref{eq:lift}.}
\State \textbf{Online} (per request):
\For{$g$ among the recalled campaigns}
    \For{$c\in\Ccal_g$}
    \parState{Set the prior mean $\mu\leftarrow\hat\mu_z$ of \eqref{eq:recon} at the finest node $z$ with support, backing off by \eqref{eq:fallback}.}
    \parState{Compute the real-time rate $\hat\mu_{c,s}$ and the support $n_{\eff}$ from the creative's marginals by \eqref{eq:recon} and \eqref{eq:neff}.}
    \State Draw the score
    \[\quad \quad \quad \quad \ \ \begin{aligned}\tilde\theta_{c,s}\sim\mathrm{Beta}\big(&\kappa_c\mu+n_{\eff}\hat\mu_{c,s},\kappa_c(1-\mu)+n_{\eff}(1-\hat\mu_{c,s})\big).\end{aligned}\]
    \EndFor
    \parState{Set the quota $E_g = \max(1,\lfloor\rho|\Ccal_g|\rfloor)$ and let $\Scal_g$ be the $E_g$ creatives of $\Ccal_g$ with the largest draws, as in \eqref{eq:topk}.}
\EndFor
\If{$\sum_g E_g > B$}
    \parState{Reduce $\{\Scal_g\}_g$ to $B$ creatives by round robin over campaigns, one slot each per pass.}
\EndIf
\parState{\Return The shortlists $\{\Scal_g\}_g$ passed to the ad ranker.}
\end{algorithmic}
\end{algorithm}

\section{Experiments}\label{sec:exp}

We evaluate Algorithm~\ref{alg:ts} by replaying logged Sponsored-ads traffic from Roblox. For each logged request we recover the full set of recalled creatives, rebuild each one's serve-time posterior $(\alpha_{c,s},\beta_{c,s})$ of~\eqref{eq:servepost} from the real-time single-attribute counters the serving path reads, and re-run every competing rule over that same candidate set. The counter state is the one the deployed system held, and nothing recorded after a request enters its scores, so the replay exercises the method on what production knew at request time rather than on a retrospective aggregate. Appendix~\ref{app:state} states how we reconstruct that state. Two questions are at issue, and they pull against each other: whether sampling the hierarchical posterior exploits the information in the counts while exploring, and whether that sampling is cheap enough to run for every recalled creative inside the request. Appendix~\ref{app:exp} collects the experimental details the main text omits.

\subsection{Data}\label{sec:exp:data}

\rev{The replay uses logged Sponsored-ads requests drawn at random from hourly partitions of production traffic, a few hundred of them, which together contribute more than $10^5$ recalled campaigns, each counted once per request that recalls it and carrying its own creative pool.} The extraction applies a single cap to each request's recalled set. At most $L=500$ campaigns are kept, drawn uniformly at random and kept with their full creative pools. This cap is ours rather than production's, and it keeps the cost of the replay proportional to the number of requests rather than to pool depth. Table~\ref{tab:data} in Appendix~\ref{app:state} reports the resulting sample; that appendix states how the replay matches production serving.

Two properties shape the results. \rev{A randomly sampled request carries just under $L$ campaigns and more than $1{,}000$ creatives}, about $2.6$ per campaign, so the pre-ranker is choosing roughly one of them, and such shallow pools bound how much any rule can win. A $16.9\%$ share of candidates have $n_{\eff}=0$, so a sixth of the traffic tests the fallback ladder~\eqref{eq:fallback} rather than the counters. The pool sizes are heavily concentrated, with over half of contested pools holding two creatives and about three in ten holding three, and a tail reaching $20$, so the quota mostly resolves to one or two slots and the mean shortlist width at $\rho=0.8$ is only $1.209$.

\subsection{Setup and Compared Methods}\label{sec:exp:setup}

A \emph{pre-ranking rule} must commit, before any ranking model runs, to a shortlist $\Scal_g\subseteq\Ccal_g$ for every recalled $g$, subject to $\sum_g|\Scal_g|\le B$. We compare five rules. \textsc{Uniform} forwards one creative per campaign, drawn at random; it needs no per-creative statistics and is the reference any pre-ranker has to beat. \textsc{TS-One} also forwards one per campaign, the best of one Thompson draw~\eqref{eq:draw} per candidate. \textsc{RR-Random} takes the quota $E_g=\max(1,\lfloor\rho|\Ccal_g|\rfloor)$ and the round-robin reduction of Algorithm~\ref{alg:ts} but fills the slots at random. \textsc{TS-RR} is Algorithm~\ref{alg:ts} in full. \textsc{Oracle} forwards every recalled creative; it more than doubles the heavy ranker's work and so cannot be served, but it upper-bounds what any pre-ranker could achieve.

The four servable rules vary two things separately, how many slots a campaign gets and which creatives fill them. The quota and the round-robin reduction depend only on the pool sizes $|\Ccal_g|$, on $\rho$ and on $B$, never on a score, so \textsc{Uniform} and \textsc{TS-One} fill an \emph{identical} allocation of slots on every request, as do \textsc{RR-Random} and \textsc{TS-RR}. Within either pair only the score differs, so any difference in quality is attributable to the posterior and not to a larger share of the ranker's budget. Between the pairs only the slot count differs. We report $\rho=0.8$ and $B=600$ throughout. Appendix~\ref{app:impl} collects the implementation details, from why this cap is the operating point that puts the round robin to work to how the quality figures and the timings are produced.

\subsection{Evaluation Metrics}\label{sec:exp:metrics}

A pre-ranker is judged not by the creative it likes best but by the set it hands on, and how that set is consumed is fixed by ad ranking, which scores the creatives it receives and keeps at most one rendering of each advertised universe, which is what the auction receives (Figure~\ref{fig:pipeline}). A shortlist is therefore worth the best creative it contains, and we take its reward to be the best per-user $\PTR$ prediction among its members,
\begin{equation}\label{eq:exprew}
R(\Scal_g)=\max_{c\in\Scal_g} \widehat{\PTR}(c,u_t),
\end{equation}
where $\widehat{\PTR}(c,u_t)$ is the play-through rate predicted by the most accurate deployed ranking model for creative $c$ and the requesting user $u_t$, logged for every recalled candidate. Appendix~\ref{app:models} describes the model and reports robustness checks; Section~\ref{sec:exp:models} measures its cost. A rule that forwards a pool's best creative attains the pool's maximum; one that drops it cannot be rescued downstream. We average \eqref{eq:exprew} over \emph{contested} pools ($|\Ccal_g|\ge2$), since on a single-creative campaign every rule returns the same creative and including such pools would dilute all contrasts by a common factor. Absolute play-through rates are small and platform-specific, so with $\bar R(\cdot)$ that average we summarize each rule $\pi$ by
\begin{equation}\label{eq:expgap}
\begin{gathered}
\mathrm{Lift}(\pi)=\frac{\bar R(\pi)}{\bar R(\textsc{Uniform})}-1,\\[2pt]
\mathrm{Gap}(\pi)=\frac{\bar R(\pi)-\bar R(\textsc{Uniform})}{\bar R(\textsc{Oracle})-\bar R(\textsc{Uniform})}.
\end{gathered}
\end{equation}
The gap is the more informative of the two: it is $0$ for \textsc{Uniform} and $1$ for forwarding everything, so it states directly what fraction of the headroom a pre-ranker recovers while still respecting the cap. Thompson draws are averaged over $200$ replications per candidate, whose Monte-Carlo standard error is at most $5.5\times10^{-7}$; sampling error over traffic is a request-clustered nonparametric bootstrap with $5{,}000$ resamples.

\subsection{The Cost of Production Ranking}\label{sec:exp:models}

The pre-ranker exists to protect the production model that follows it, which we run locally at the version serving at the time of the replay. We use the platform's most accurate deployed ranking model both to evaluate selection quality and to measure downstream cost. \conf{Most of its inputs are per-candidate}, so each creative enters through many distinct feature paths and its cost grows in proportion to the shortlist. The cap $B$ therefore protects this model: \rev{a 
randomly sampled request carries more than $1{,}000$ creatives on average}, so forwarding the full pool more than doubles its work against the shortlists we report. Appendix~\ref{app:models} describes the model and reports an additional quality check.

\subsection{Results}\label{sec:exp:perf}

\paragraph{Selection quality.} Table~\ref{tab:main} reports predicted play-through rate under the most accurate deployed ranking model. Extra slots alone are worth little. \textsc{RR-Random} forwards the $600$ creatives the cap allows, a $21\%$ increase in ranker load over \textsc{Uniform}, and returns $+2.80\%$, closing $15.3\%$ of the gap. The posterior alone is worth roughly twice as much at no extra ranker load: \textsc{TS-One} forwards exactly as many creatives as \textsc{Uniform}\rev{, one per recalled campaign,} and returns $+6.09\%$, closing $33.2\%$. Combining them, Algorithm~\ref{alg:ts} returns $+8.14\%$ ($95\%$ CI $[7.64,8.65]$) and closes $44.3\%$ of the distance to forwarding everything. All bootstrap intervals exclude zero and, more to the point, those for \textsc{TS-RR} and \textsc{RR-Random} are disjoint, so at an identical slot allocation the posterior is doing work that extra capacity cannot buy. Ranking by the posterior mean instead of one draw returns $+9.51\%$ (\textsc{Mean-RR}, Table~\ref{tab:main}), so exploration costs $1.37$ points of lift in this frozen replay; Appendix~\ref{app:explore} gives the decomposition and why the replay understates its value. Averaging over requests rather than over pools leaves every lift unchanged to two decimal places, and the closed-form expectations available for \textsc{Uniform} and \textsc{RR-Random} agree with the sampler to $3.9\times10^{-5}$ and $3.1\times10^{-5}$ relative.

\begin{table}[t] \centering \caption{Predicted play-through rate on contested pools, $\rho=0.8$, $B=600$. Lifts are against \textsc{Uniform} with $95\%$ request-clustered bootstrap intervals; lift and gap follow~\eqref{eq:expgap}. \emph{Forwarded} is the mean creatives per request reaching the ranker.} \label{tab:main}
\begin{tabular}{lrcr}
\toprule
Method & Forwarded & Lift ($95\%$ CI) & Gap \\
\midrule
\textsc{Uniform}   & \rev{$\approx\!L$}     & ---                       & $0.0\%$ \\
\textsc{RR-Random} & $600.0$     & $+2.80\ [2.71,\,2.90]$    & $15.3\%$ \\
\textsc{TS-One}    & \rev{$\approx\!L$}     & $+6.09\ [5.62,\,6.58]$    & $33.2\%$ \\
\textsc{TS-RR}     & $600.0$     & $+8.14\ [7.64,\,8.65]$    & $44.3\%$ \\
\abl{\textsc{Mean-RR}}  & \abl{$600.0$} & \abl{$+9.51\ [8.87,\,10.18]$} & \abl{$51.8\%$} \\
\textsc{Oracle}    & \rev{$>1{,}000$} & $+18.36\ [17.78,\,18.97]$ & $100\%$ \\
\bottomrule
\end{tabular}
\end{table}

\paragraph{Computation time.} Table~\ref{tab:time} reports per-request cost. Scoring every recalled creative takes $79$\,\textmu s, and the full pre-ranker, scoring and shortlisting together, takes $2.2$\,ms. The heavy ran
ker dominates total cost and grows with the shortlist: forwarding $100$ more creatives adds $6.4$\,ms. The pre-ranker therefore costs less than the ranker spends on those extra creatives.

\begin{table}[t] \centering \caption{Per-request computation time. \emph{Score} and \emph{select} are the pre-ranker's two timed phases, and \emph{total} adds the ranker the cap protects. \textsc{Oracle}'s ranker entry is extrapolated past the largest shortlist measured.} \label{tab:time}
\begin{tabular}{@{}lrrrrr@{}}
\toprule
& & \multicolumn{2}{c}{Pre-ranker ($\mu$s)} & \conf{Ranker} & Total \\
\cmidrule(lr){3-4}
Method & Fwd. & Score & Select & \conf{(ms)} & (ms) \\
\midrule
\textsc{Uniform}   & \rev{$\approx\!L$}   & $0.1$  & $14.8$      & \conf{$33.0$} & $33.0$ \\
\textsc{RR-Random} & $600.0$   & $6.6$  & $2{,}109.9$ & \conf{$39.4$} & $41.5$ \\
\textsc{TS-One}    & \rev{$\approx\!L$}   & $78.0$ & $2{,}071.0$ & \conf{$33.0$} & $35.1$ \\
\textsc{TS-RR}     & $600.0$   & $79.2$ & $2{,}122.5$ & \conf{$39.4$} & $41.6$ \\
\midrule
\textsc{Oracle}    & \rev{$>1{,}000$} & $0.1$ & $3.0$     & \conf{$83.1$} & $83.1$ \\
\bottomrule
\end{tabular}
\end{table}

Read down the last column: \textsc{TS-RR} costs $41.6$\,ms against \textsc{RR-Random}'s $41.5$\,ms, a $0.2\%$ difference, for $+8.14\%$ against $+2.80\%$ in predicted play-through rate. Nearly triple the quality gain for a fifth of a percent more compute is the practical case for the method, and \textsc{TS-One} sharpens it: at $35.1$\,ms it costs $15\%$ \emph{less} than the score-free round robin, forwards no more creatives than \textsc{Uniform}, and still doubles its lift. \textsc{Oracle}'s $+18.36\%$, by contrast, would require $83.1$\,ms and \rev{the entire recalled pool} at the ranker, $2.0\times$ the cost of Algorithm~\ref{alg:ts}; recovering $44\%$ of that headroom on more than $40\%$ fewer creatives is what the pre-ranker buys.

\paragraph{Sensitivity.} Table~\ref{tab:sweep} sweeps the quota fraction against the latency cap under another deployed ranking model, described in Appendix~\ref{app:models}. It is the cheaper of the two and the one we use as a robustness check, so the sweep both varies the operating point and repeats it away from the headline model. When the cap binds hard the round robin has no slack and \textsc{RR-Random} returns essentially nothing while the posterior still delivers $+3.34\%$, so exploration is what survives a tight budget; as the cap loosens both improve and the posterior keeps a roughly constant margin; and at $\rho=1.0$ uncapped both coincide at $+10.31\%$, reproducing that model's \textsc{Oracle} row. The ranking of methods is unchanged from Table~\ref{tab:main}.

\begin{table}[t] \centering \caption{Quota fraction $\rho$ against latency cap $B$ under another deployed ranking model, with $50$ draws per candidate. Lifts are against \textsc{Uniform} and \emph{gap} is \textsc{TS-RR}'s, as in~\eqref{eq:expgap}. \rev{\emph{Forwarded} is in multiples of the one-creative-per-campaign shortlist.}} \label{tab:sweep}
\begin{tabular}{@{}rrrrrr@{}}
\toprule
$\rho$ & $B$ & \rev{Fwd.\ ($\times L$)} & \textsc{RR-Random} & \textsc{TS-RR} & Gap \\
\midrule
$0.2$ & $600$      & \rev{$1.00$}   & $+0.02\%$ & $+3.34\%$ & $32.5\%$ \\
$0.4$ & $600$      & \rev{$1.03$}   & $+0.26\%$ & $+3.52\%$ & $34.2\%$ \\
$0.6$ & $600$      & \rev{$1.15$}   & $+1.14\%$ & $+4.17\%$ & $40.5\%$ \\
$0.8$ & $500$      & \rev{$1.01$}   & $+0.04\%$ & $+3.39\%$ & $33.0\%$ \\
$0.8$ & $600$      & \rev{$1.21$}   & $+1.62\%$ & $+4.52\%$ & $44.0\%$ \\
$0.8$ & $800$      & \rev{$1.57$}   & $+4.09\%$ & $+6.29\%$ & $61.2\%$ \\
$0.8$ & $\infty$   & \rev{$1.64$}   & $+4.40\%$ & $+6.48\%$ & $63.0\%$ \\
$1.0$ & $600$      & \rev{$1.21$}   & $+1.52\%$ & $+4.42\%$ & $42.9\%$ \\
$1.0$ & $800$      & \rev{$1.61$}   & $+4.71\%$ & $+6.65\%$ & $64.5\%$ \\
$1.0$ & $\infty$   & \rev{$2.64$}   & $+10.31\%$ & $+10.31\%$ & $100\%$ \\
\bottomrule
\end{tabular}
\end{table}

\section{Conclusion}\label{sec:conclusion}

We studied creative pre-ranking, a stage that must score every recalled creative inside the request and that, by choosing what the ranker sees, also chooses what the system will ever observe. Our method estimates how much a sparse creative should borrow from its parents rather than tuning it, reconstructs the per-segment posterior from counters that stay dense, one attribute at a time, rather than from the joint counters over attribute combinations, which grow exponentially and are hard to maintain, and selects by a single Thompson draw. On replayed Roblox traffic it recovers close to three times the improvement an equally sized random shortlist obtains, at $0.2\%$ additional serving cost: the exploration a pre-ranker needs is negligible beside the per-candidate cost of the ranker it protects. After an A/B test on live traffic, it now serves production at Roblox. The same construction applies wherever a ranker sees a shortlist drawn from sparse, hierarchically organized items, such as cold-start items in web search and recommendation. Learning the residual left by the linear log-odds reconstruction and allocating slots by expected information gain are natural next steps.



\newpage

\newpage

\baselineskip=17pt \bibliography{dco_final}

\newpage

\appendix

\section{Proofs and Offline Fitting}\label{app:proofs}

\subsection{Proof of Theorem~\ref{thm:mse}}\label{app:mse} \begin{proof} Throughout, $\theta_i$ is a \emph{random variable} (its randomness is the between-child variation posited in~\eqref{eq:remodel}), and every expectation is taken over both $\theta_i$ and, given $\theta_i$, the Binomial draw $y_i$. We use two conditional facts: given $\theta_i$, $y_i\sim\mathrm{Binomial}(n_i,\theta_i)$ has $\E[y_i\mid\theta_i]=n_i\theta_i$ and $\Var(y_i\mid\theta_i)=n_i\theta_i(1-\theta_i)$, so $\hat p_i=y_i/n_i$ satisfies
\begin{equation}\label{eq:condmoments}
\E[\hat p_i\mid\theta_i]=\theta_i,\qquad
\Var(\hat p_i\mid\theta_i)=\frac{\theta_i(1-\theta_i)}{n_i}.
\end{equation}

Write the error as $\hat\theta_i-\theta_i=w(\hat p_i-\theta_i)+(1-w)(\mu-\theta_i)$ and expand:
\begin{align*}
\E[(\hat\theta_i-\theta_i)^2]
&= w^2\,\E[(\hat p_i-\theta_i)^2]\\
&\quad+2w(1-w)\,\E[(\hat p_i-\theta_i)(\mu-\theta_i)]\\
&\quad+(1-w)^2\,\E[(\mu-\theta_i)^2].
\end{align*}
The cross term vanishes: conditionally on $\theta_i$ the factor $(\mu-\theta_i)$ is fixed and $\E[\hat p_i-\theta_i\mid\theta_i]=0$ by \eqref{eq:condmoments}, so the tower rule gives $\E[(\hat p_i-\theta_i)(\mu-\theta_i)] =\E\big[(\mu-\theta_i)\,\E[\hat p_i-\theta_i\mid\theta_i]\big]=0$. The same zero conditional mean makes the conditional second moment of $\hat p_i-\theta_i$ equal its conditional variance $\theta_i(1-\theta_i)/n_i$, so averaging over $\theta_i$ and using $\E[\theta_i^2]=\Var(\theta_i)+(\E\theta_i)^2=\sigma^2+\mu^2$,
\begin{align*}
\E[(\hat p_i-\theta_i)^2]
&=\frac{\E[\theta_i(1-\theta_i)]}{n_i}
=\frac{\E[\theta_i]-\E[\theta_i^2]}{n_i}\\
&=\frac{\mu-(\sigma^2+\mu^2)}{n_i}
=\frac{\mu(1-\mu)-\sigma^2}{n_i},
\end{align*}
while the remaining term is $\E[(\mu-\theta_i)^2]=\Var(\theta_i)=\sigma^2$. Hence \[ \E[(\hat\theta_i-\theta_i)^2] =w^2\,\frac{\mu(1-\mu)-\sigma^2}{n_i}+(1-w)^2\sigma^2, \] a convex quadratic in $w$ whose derivative vanishes at
\begin{align*}
w^\star&=\frac{\sigma^2}{\sigma^2+(\mu(1-\mu)-\sigma^2)/n_i}=\frac{n_i}{n_i+\kappa},\\
\kappa&=\frac{\mu(1-\mu)-\sigma^2}{\sigma^2}=\frac{\mu(1-\mu)}{\sigma^2}-1,
\end{align*}
which is \eqref{eq:wstar}; here $\kappa\ge0$ because any $[0,1]$-valued rate satisfies $\sigma^2=\Var(\theta_i)\le \mu(1-\mu)$. Substituting $w^\star$ into \eqref{eq:shrink} gives $\hat\theta_i=w^\star\hat p_i+(1-w^\star)\mu=(\kappa\mu+y_i)/(\kappa+n_i)$. This completes the proof of Theorem~\ref{thm:mse}. \end{proof}

\subsection{The Scatter Decomposition}\label{app:scatter}

\begin{theorem}[Impression-weighted scatter decomposition]\label{thm:scatter} For a parent with children $[C]=\{1,\dots,C\}$, let $N=\sum_{i=1}^{C} n_i$, let $\hat\mu=\big(\sum_{i=1}^{C} y_i\big)/N$ be the pooled rate, and let $Q=\sum_{i=1}^{C} n_i(\hat p_i-\hat\mu)^2$ be the impression-weighted scatter of the children. Under Assumption~\ref{ass:re} with parent mean $\mu$ and variance $\sigma^2$, the expectation of $Q$ is given by~\eqref{eq:scatter}. \end{theorem}

\begin{proof} Recall that $\E[\hat p_i]=\mu$ for every child. By the law of total variance and the conditional moments \eqref{eq:condmoments},
\begin{align*}
\Var(\hat p_i)&=\Var(\theta_i)+\E[\Var(\hat p_i\mid\theta_i)]\\
&=\sigma^2+\frac{\E[\theta_i(1-\theta_i)]}{n_i}
=\sigma^2+\frac{\mu(1-\mu)-\sigma^2}{n_i},
\end{align*}
using $\E[\theta_i(1-\theta_i)]=\mu(1-\mu)-\sigma^2$ as in Theorem~\ref{thm:mse}. With $\hat\mu=\big(\sum_i n_i\hat p_i\big)/N$, the weighted identity $Q=\sum_i n_i(\hat p_i-\mu)^2-N(\hat\mu-\mu)^2$ gives $\E[Q]=\sum_i n_i\Var(\hat p_i)-N\Var(\hat\mu)$, and it remains to evaluate the two terms. The first is $\sum_i n_i\Var(\hat p_i)=\sum_i\big[\sigma^2 n_i+\mu(1-\mu)-\sigma^2\big] =\sigma^2 N+C\big(\mu(1-\mu)-\sigma^2\big)$. For the second, the children are independent by Assumption~\ref{ass:re}, so \[ \Var(\hat\mu)=\frac1{N^2}\sum_i n_i^2\Var(\hat p_i) =\frac{\sigma^2\sum_i n_i^2+N\big(\mu(1-\mu)-\sigma^2\big)}{N^2}, \] hence $N\Var(\hat\mu)=\sigma^2\big(\sum_i n_i^2\big)/N+\mu(1-\mu)-\sigma^2$. Subtracting, the $\mu(1-\mu)-\sigma^2$ terms leave $C-1$ copies and the $\sigma^2$ terms leave $N-\sum_i n_i^2/N$: \[ \E[Q]=(C-1)\big(\mu(1-\mu)-\sigma^2\big)+\sigma^2\Big(N-\frac{\sum_i n_i^2}{N}\Big), \] which is \eqref{eq:scatter}. This completes the proof of Theorem~\ref{thm:scatter}. \end{proof}

\subsection{Proof of Corollary~\ref{cor:sigma}}\label{app:sigma} \begin{proof} Fix a parent $j$. Its children satisfy Assumption~\ref{ass:re} with the parent's own mean $\mu_j$ and the shared variance $\sigma^2$, so Theorem~\ref{thm:scatter} applies. Collecting the $\sigma^2$ terms in \eqref{eq:scatter},
\begin{align*}
&(C_j-1)\big(\mu_j(1-\mu_j)-\sigma^2\big)+\sigma^2\Big(N_j-\tfrac{\sum_i n_{ji}^2}{N_j}\Big)\\
&\qquad=(C_j-1)\,\mu_j(1-\mu_j)+\sigma^2 D_j,\\
&\qquad D_j=N_j-\frac{\sum_{i=1}^{C_j} n_{ji}^2}{N_j}-(C_j-1),
\end{align*}
so $\E[Q_j]=(C_j-1)\mu_j(1-\mu_j)+\sigma^2 D_j$. Summing over the $P$ parents, whose scatters are independent because the children are independent across parents, \[ \E\Big[\textstyle\sum_{j=1}^{P} Q_j\Big] =\sum_{j=1}^{P}(C_j-1)\,\mu_j(1-\mu_j)+\sigma^2\sum_{j=1}^{P} D_j . \] Since $\sum_j D_j$ is deterministic, the known-mean estimator $\tilde\sigma^2=\big(\sum_j Q_j-\sum_j(C_j-1)\mu_j(1-\mu_j)\big)/\sum_j D_j$ has $\E[\tilde\sigma^2]=\sigma^2$, i.e.\ it is unbiased. The implemented form~\eqref{eq:sigmahat} replaces each unknown $\mu_j$ by $\hat\mu_j$ and clips at $0$; since $\E[\hat\mu_j(1-\hat\mu_j)]=\mu_j(1-\mu_j)-\Var(\hat\mu_j)$ with $\Var(\hat\mu_j)=O(1/N_j)$, this adds an $O(1/N_j)$ bias per parent, and the clip acts only when $\hat\sigma^2$ would be negative; both vanish at the impression volumes where the prior is fit, leaving $\hat\sigma^2$ consistent for $\sigma^2$. This completes the proof of Corollary~\ref{cor:sigma}. \end{proof}

\subsection{Independence and Duplication}\label{app:limits}

\begin{proposition}[Independence and duplication]\label{prop:limits} Suppose the joint log-odds are additive across attributes, $\logit\theta_z(s)=\eta_z+\sum_{k}\gamma_{z,k}(d_k)$, and that each bucket is measured without error by the exposure-weighted average of the log-odds of the impressions it contains. \begin{enumerate}[label=(\roman*),leftmargin=*] \item If the attributes are mutually independent in the node's traffic, then \eqref{eq:pool} holds with $\psi\equiv\mathbf 1$, so \eqref{eq:naive} is exact. \item If $J$ of the attributes are duplicates of one another in the node's traffic, and the remaining $K-J$ are independent of them and of each other, then \eqref{eq:pool} holds with $\psi_k=1/J$ on the duplicated group and $\psi_k=1$ elsewhere. \end{enumerate} \end{proposition}

\begin{proof} Let $\bar\gamma_{z,j}$ denote the exposure-weighted average of $\gamma_{z,j}(d_j)$ over the node's impressions, and write $E_k[\cdot\mid v]$ for the exposure-weighted average over the impressions with $d_k=v$, i.e.\ over bucket $(k,v)$. By hypothesis the measured log-odds of a bucket is the average of the log-odds it contains, so by additivity
\begin{equation}\label{eq:bucketavg}
\begin{gathered}
\logit\bar p_{z,k}(v)=\eta_z+\gamma_{z,k}(v)+\sum_{j\ne k}E_k\big[\gamma_{z,j}(d_j)\mid v\big],\\
\logit\hat p_z=\eta_z+\sum_{j=1}^{K}\bar\gamma_{z,j},
\end{gathered}
\end{equation}
where attribute $k$ is constant on its own bucket, hence contributes $\gamma_{z,k}(v)$ exactly.

\emph{(i)} If the attributes are mutually independent in the node's traffic, then conditioning on $d_k=v$ does not change the distribution of $d_j$ for $j\ne k$, so $E_k[\gamma_{z,j}(d_j)\mid v]=\bar\gamma_{z,j}$. Subtracting the two identities in \eqref{eq:bucketavg} gives $\delta_{z,k}(v)=\gamma_{z,k}(v)-\bar\gamma_{z,k}$: each lift isolates its own attribute's centered effect. Summing over the attributes of $s=(d_1,\dots,d_K)$ and adding $\logit\hat p_z$,
\begin{align*}
&\logit\hat p_z+\sum_{k=1}^{K}\delta_{z,k}(d_k)\\
&\quad=\eta_z+\sum_{j=1}^{K}\bar\gamma_{z,j}+\sum_{k=1}^{K}\big(\gamma_{z,k}(d_k)-\bar\gamma_{z,k}\big)\\
&\quad=\eta_z+\sum_{k=1}^{K}\gamma_{z,k}(d_k)=\logit\theta_z(s),
\end{align*}
which is \eqref{eq:pool} with $\psi\equiv\mathbf 1$.

\emph{(ii)} Let $G$ with $|G|=J$ be the duplicated group. Duplication means that within the node's traffic the values $\{d_j\}_{j\in G}$ determine one another, so they are indexed by a single common value $v$; write $\Gamma_z(v)=\sum_{j\in G}\gamma_{z,j}(d_j(v))$ for the group's total effect at that value and $\bar\Gamma_z$ for its exposure-weighted average. For $k\in G$, bucket $(k,v)$ consists exactly of the impressions whose group value is $v$, so the group contributes $\Gamma_z(v)$ there, while the remaining attributes, being independent of the group, contribute their unconditional averages; \eqref{eq:bucketavg} therefore gives \[ \delta_{z,k}(v)=\Gamma_z(v)-\bar\Gamma_z\quad\text{for every }k\in G, \] the \emph{same} value for all $J$ members, which is the double counting that \eqref{eq:naive} incurs. For $k\notin G$ the argument of part~(i) applies unchanged and yields $\delta_{z,k}(v)=\gamma_{z,k}(v)-\bar\gamma_{z,k}$. Hence, weighting the group by $1/J$ and the rest by $1$,
\begin{align*}
&\logit\hat p_z+\sum_{k\in G}\frac1J\,\delta_{z,k}(d_k)+\sum_{k\notin G}\delta_{z,k}(d_k)\\
&\quad=\eta_z+\Gamma_z(v)+\sum_{k\notin G}\gamma_{z,k}(d_k)
=\logit\theta_z(s),
\end{align*}
since $\logit\hat p_z=\eta_z+\bar\Gamma_z+\sum_{k\notin G}\bar\gamma_{z,k}$. This is \eqref{eq:pool} with the stated $\psi$ and completes the proof of Proposition~\ref{prop:limits}. \end{proof}

\subsection{Fitting the Pooling Weights}\label{app:psifit}

Fisher scoring solves \eqref{eq:psifit} with group-by sums alone: each pass accumulates $\sum n\pi(1-\pi)\,xx\transp$ and $\sum x\,(y-n\pi)$ over cells and applies one $K\times K$ solve. Iterating rather than stopping at the one-step least-squares version matters, because saturated cells carry extreme targets at ample impression weight and on our logs push the weakly identified attributes out of $[0,1]$, whereas the Binomial weights $n\pi(1-\pi)$ discount exactly those cells.

\section{Experimental Details}\label{app:exp}

\subsection{Serve-Time State}\label{app:state}

Production scores recalled creatives with the online steps of Algorithm~\ref{alg:ts}, and the replay runs those same steps on each logged candidate. Those steps read two sources, the hierarchical table published offline and the creative's own counters maintained online.

The hierarchical table is a batch product, and it stores parameters. A row gives one node above the creative its rate and its shrunk marginal lifts from \eqref{eq:mshrink}--\eqref{eq:lift}, each fit on a window of several days of impressions, and the same publication carries the concentrations of Corollary~\ref{cor:sigma} and the pooling weights $\hat\psi$ of \eqref{eq:psifit}. The pipeline republishes all of it on a regular schedule, so what a request reads was computed hours before that request arrived. The lag costs little, because every one of these quantities pools over many creatives and drifts slowly. The creative's own single-attribute counters move faster. They cover a short trailing window and advance with every impression the creative receives, so they carry what it has just done and are the only fast-moving input to the posterior. This is the split the method is built around, a slow prior refreshed offline and fast evidence refreshed online.

The replay reads both sources as serving had them. For every candidate it takes the counter values and the published row available at that moment, never a later state, so a creative is scored on the evidence its own recent traffic had supplied by then. Logged ranker predictions supply the rewards. Table~\ref{tab:data} summarizes the replay sample.

\begin{table}[t] \centering \caption{\rev{The replay sample, as shares and per-request or per-campaign rates rather than totals. A pool is a (request, campaign) pair, contested when the campaign has at least two recalled creatives.}} \label{tab:data}
\begin{tabular}{lr}
\toprule
\rev{Recalled pools that are contested ($|\Ccal_g|\ge2$)} & \rev{$99.1\%$} \\
\rev{Campaigns per request, mean}               & \rev{just under $L$} \\
\rev{Creatives per request, mean}               & \rev{$>1{,}000$} \\
Creatives per campaign, mean              & $2.637$ \\
Creatives per campaign, max               & $20$ \\
Cold-start candidates ($n_{\eff}=0$)      & \rev{$16.9\%$} \\
\bottomrule
\end{tabular}
\end{table}

\subsection{Operating Point and Implementation Details}\label{app:impl}

We report $\rho=0.8$ and $B=600$ as the headline configuration, and the choice of $B$ is deliberate. Two numbers bound the useful range. The lower one is a request's count of recalled campaigns, at most the sampling cap $L$ of Section~\ref{sec:exp:data}; the upper one is its uncapped quota $\sum_g E_g$, which at $\rho=0.8$ is the larger of the two. A cap above the quota would never trigger, leaving the round robin untested, while a cap below the campaign count would deny some campaign its only slot. Every configuration we report takes $B\ge L$, and $B=600$ falls between the two numbers, so the reduction binds while the guarantee $|\Scal_g|\ge1$ of Algorithm~\ref{alg:ts} still holds and the round robin only ever redistributes second and later slots, the regime it is designed for. In deployment $B$ follows the latency budget allotted to the ranker.

Quality is averaged over $200$ Monte-Carlo replications of the Thompson draws and bootstrapped with $5{,}000$ request-clustered resamples, all from a fixed seed. Timings are the mean over $25$ repeats per request. The pre-ranker is a NumPy reference implementation; the two rules that need no ranking of scores bypass the per-campaign selection loop entirely, which is why their shortlist phase is faster by more than the algorithm alone would suggest.

All timings are single-process wall-clock measurements taken on \conf{a two-socket x86-64 Linux server, with the ranking models running as CPU builds limited to $8$ threads}; \conf{production serves them on GPU}, which shifts the absolute level without changing the quantity the argument uses, namely how each model's cost \emph{scales} in the number of creatives forwarded. That scaling follows from how a candidate enters the model, so it carries over to the serving path. The ranker measurements time the forward pass on inputs already laid out as tensors, as in production, where features are assembled before the model runs.

\subsection{Production Ranking Models and Robustness}\label{app:models}

The main text uses the platform's most accurate deployed model, L2, for both quality and timing. \conf{L2 takes mostly per-candidate inputs: it bucketizes and embeds them, models user--item interaction, and uses aggregate features of the advertised experience covering engagement, quality and retention, monetization, multiplayer and chat activity, and audience composition.} It is a deep click-through-rate model of the now-standard kind \citep{wang2017deep, zhou2018deep}; because a candidate enters through many feature paths, its cost grows in proportion to the shortlist. It also deduplicates, emitting at most one creative per advertised universe, which is why reward~\eqref{eq:exprew} is a maximum over shortlist members.

We repeat the quality evaluation with a second deployed model, L1. \conf{Its inputs are almost entirely request- and user-level: personalization history over several windows and request context. Exactly \emph{one} is per-candidate, a hash of the ad identifier}, so the creative enters only through an embedding lookup. It is a deterministic point predictor with no exploration, so we use it only as a second quality check, not as a substitute for the pre-ranker. Despite its different architecture, Table~\ref{tab:robustness} preserves the ordering of every method and gives Algorithm~\ref{alg:ts} essentially the same gap closure as the main evaluation ($44.0\%$ against $44.3\%$).

\begin{table}[t] \centering \caption{Robustness under L1. The setup matches Table~\ref{tab:main}.} \label{tab:robustness}
\begin{tabular}{lrcr}
\toprule
Method & Forwarded & Lift ($95\%$ CI) & Gap \\
\midrule
\textsc{Uniform} & \rev{$\approx\!L$} & --- & $0.0\%$ \\
\textsc{RR-Random} & $600.0$ & $+1.62\ [1.51,\,1.72]$ & $15.7\%$ \\
\textsc{TS-One} & \rev{$\approx\!L$} & $+3.35\ [3.01,\,3.69]$ & $32.5\%$ \\
\textsc{TS-RR} & $600.0$ & $+4.54\ [4.16,\,4.91]$ & $44.0\%$ \\
\abl{\textsc{Mean-RR}} & \abl{$600.0$} & \abl{$+5.10\ [4.68,\,5.52]$} & \abl{$49.5\%$} \\
\textsc{Oracle} & \rev{$>1{,}000$} & $+10.30\ [9.68,\,10.92]$ & $100\%$ \\
\bottomrule
\end{tabular}
\end{table}

\subsection{The Price of Exploring}\label{app:explore}

\abl{Algorithm~\ref{alg:ts} explores by ranking on a posterior draw, and the \textsc{Mean-RR} row of Table~\ref{tab:main} prices that draw. It fills the same score-independent allocation from the same reconstructed posterior, ranking by the mean $\alpha_{c,s}/(\alpha_{c,s}+\beta_{c,s})$ instead of by one sample, and returns $+9.51\%$ against $+8.14\%$: exploration costs $1.37$ points of lift, while both rules beat the score-free round robin threefold. The gain thus belongs to the posterior, and the draw rides on top of it for little. That $1.37$ is also the whole cost with none of the benefit. Exploration pays off by turning impressions into evidence, but the replay is frozen: a creative's posterior is the same whether or not the draw put it on screen. Online it would not be, so $+8.14\%$ is a lower bound on what the sampled rule is worth.}

\end{document}